%% file: isometry_conference_13_06.tex
\documentclass[twocolumn]{IEEEtran}

\usepackage{amsfonts,amsmath,graphicx,cite,bm,amssymb,amsthm,enumerate,epsfig,psfrag}
\include{ilyas_commands}

\usepackage{fontenc}
\usepackage{inputenc}
\usepackage{authblk}

\begin{document}

\title{Group Symmetry and non-Gaussian Covariance Estimation}

\author{Ilya Soloveychik}
\author{Ami Wiesel}
\affil{Selim and Rachel Benin
School of Computer Science and Engineering, Hebrew University of Jerusalem, Israel}
\affil{ilya.soloveychik@mail.huji.ac.il, amiw@cs.huji.ac.il}

\maketitle

\begin{abstract}
We consider robust covariance estimation with group symmetry constraints. Non-Gaussian
covariance estimation, e.g.,
Tyler scatter estimator and Multivariate Generalized Gaussian distribution methods,
usually involve non-convex
minimization problems. Recently, it was shown that the underlying principle behind their
success is an extended form
of convexity over the geodesics in the manifold of positive definite matrices.
A modern approach to improve estimation accuracy is to exploit prior knowledge via
additional constraints, e.g., restricting the attention to specific classes
of covariances which adhere to prior symmetry structures. In this paper, we prove that
such group symmetry constraints
are also geodesically convex and can therefore be incorporated into various non-Gaussian
covariance estimators.
Practical examples of such sets include: circulant, persymmetric and complex/quaternion
proper structures.
We provide a simple numerical technique for finding maximum likelihood estimates under such
constraints,
and demonstrate their performance advantage using synthetic experiments.
\end{abstract}


\section{Introduction}
Covariance estimation is a fundamental problem in the field of statistical signal
processing.
Many algorithms for detection and estimation rely on accurate covariance matrix estimation
\cite{covest_imp1,covest_imp2}.
Roughly speaking, the problem is tractable as long as the global maximum likelihood
solution can be efficiently found (or approximated). Thus, it is important to understand
whether the associated negative-log-likelihood minimization problem is convex. Following this line of thought, we
combine two ideas.
First, there is an increasing interest in covariance estimation in non-Gaussian
distributions which are typically non-convex
but have been shown to be geodesically convex \cite{ami1,ami3}. Second, many problems
adhere to known symmetry constraints which can be exploited in the estimation.
Recently, \cite{venkat} addressed such structures in the Gaussian setting. In this paper,
we will consider them in non-Gaussian covariance estimation using the theory of geodesic
convexity.

In many applications, the assumption of normal data is not realistic \cite{abr,pascal1}.
In such scenarios, improved performance may be obtained by resorting to more general
distributions, such as Generalized Gaussian and Elliptical distributions
\cite{ellipt,ami4}. The associated Maximum Likelihood optimization usually do not lead to closed form
solutions and iterative algorithms are required \cite{ami1,pascal1}.
One of the most prominent robust methods is the Tyler's method for covariance
matrix estimation in scaled Gaussian models, which has been successfully applied to
different practical applications ranging from array processing to sensor networks  \cite{tyler}.
It has been extended to other settings involving
regularization and incomplete data \cite{abr} - \cite{ami2}.
Recently, it was shown that the underlying principle behind these successful non-convex
optimizations is the geodesic convexity \cite{rapcsak, ami}. This principle provides more insight on the analysis and design of robust covariance estimation
methods, and paves the road to numerous extensions based on $g$-convexity, e.g.,
regularization \cite{ami1} and their combination with Kronecker structures \cite{ami}.


Over the last years, many works have been developed in the area of estimating covariance matrices
possessing some additional knowledge such as sparsity or structure \cite{levina}. Our work is motivated by \cite{venkat} which considered group symmetry structures. In particular, \cite{venkat} addressed symmetry constraints in random fields
of physical phenomena, Bayesian models and cyclostationary processes. In addition, it is well known that
circulant matrices are invariant to shifts \cite{dembo,circ}. Symmetric persymmetric (bisymmetric) matrices are
invariant under the exchange-operator \cite{pascal_per,dimaio}. Proper complex normal
distributions are defined via their invariance to rotations with respect to the real and
imaginary axis \cite{kay}. Proper quaternion distributions follow invariances with respect
to isoclinic rotations \cite{quat_properness,alba}.  All of these properties have been
successfully exploited in covariance estimation in the multivariate Gaussian distribution.
Many of them have also been considered in non-Gaussian distributions via problem-specific
fixed point iterations and algorithm-dependent existence, uniqueness and convergence
proofs.

%

The main result in this paper is that the set of positive definite matrices which are
invariant under a conjugation action of a subgroup of orthogonal transformations is
$g$-convex on their respective manifold. Together with the $g$-convexity of various
non-Gaussian negative-log-likelihoods,
this implies that the global constrained maximum likelihood solution can be efficiently
found using standard descent algorithms. This provides a unified framework for robust
covariance estimation with group symmetry constraints.
Unlike previous approaches, our results are not specific to any distribution, symmetry set
or even numerical algorithm. As a byproduct, we provide a few results on
specific symmetry groups and reformulate proper complex and quaternion structures using
a finite number of rotation-invariant constraints. For completeness, we also propose a
simple numerical method for solving these problems, although we emphasize that other
descent algorithms can be used instead. Finally, we demonstrate the performance advantage
of our framework via synthetic simulations in a non-Gaussian proper quaternion
environment.

The paper is organized in the following form. First, we give an outline of $g$-convexity and matrix group symmetry. Then the main result is formulated and examples of symmetry matrix classes are given. Finally, we provide a computational algorithm and numerical results.

\section{Geodesic Convexity}
Geodesic convexity is a generalization of the notion of convexity in linear spaces.
We therefore begin with a brief review on $g$-convexity on the manifold $\mathcal{P}(p)$
of positive definite matrices $p \times p$.
More details are available in \cite{rapcsak}, \cite{ami2}. With each $\Q_0,\Q_1\in
\mathcal{P}(p)$ we associate the following geodesic
\begin{eqnarray}\label{geodesics}
 \Q_t=\Q_0^{\frac{1}{2}}\(\Q_0^{-\frac{1}{2}}\Q_1\Q_0^{-\frac{1}{2}}\)^t\Q_0^{\frac{1}{2}},\quad
 t\in[0,1].
\end{eqnarray}

\begin{definition}
 A set ${\mathcal{N}}\in \mathcal{P}(p)$ is $g$-convex if for any $\Q_0,\Q_1\in
 \mathcal{N}$ the geodesic $\Q_t$ lies in $\mathcal{N}$.
\end{definition}

\begin{definition}
Given a $g$-convex subset $\mathcal{N} \subset \mathcal{P}(p)$,
we say that a function $f$ is $g$-convex on $\mathcal{N}$ if for any two points
$\Q_0,\Q_1 \in \mathcal{N}, f(\Q_t)\leq tf(\Q_0)+(1-t)f(\Q_1), \forall t \in [0,1]$.
\end{definition}

The advantage of $g$-convexity stems from the following result \cite{ami2}
\begin{prop}
 Any local minimum of a $g$-convex function over a $g$-convex set is a global minimum.
 \label{general_convexity}
\end{prop}
Finding local minimum is usually easy and hence $g$-convexity guarantees that a global
solution can also be efficiently found.

Recently, it was shown that the negative-log-likelihoods of many popular non-Gaussian
distributions are $g$-convex. Two examples are:
\begin{itemize}
 \item Tyler's \cite{ami1}
\begin{eqnarray}\label{ml1}
 L\(\{\mathbf{x}_i\}_{i=1}^n;\Q\)=\frac{p}{n}\sum_{i=1}^n \log(\mathbf{x}_i^T
 \Q^{-1}\mathbf{x}_i) + \log|\Q|,
\end{eqnarray}
 \item Mutlivariate Generalized Gaussian Distribution \cite{ami3}
 \begin{eqnarray}\label{ml2}
 L\(\{\mathbf{x}_i\}_{i=1}^n;\Q\)= \frac{1}{n} \sum_{i=1}^n (\mathbf{x}_i^T
 \Q^{-1}\mathbf{x}_i)^\beta + \log|\Q|,
\end{eqnarray}
\end{itemize}
where $\beta$ is the shape parameter.

Together with Proposition \ref{general_convexity} above, \cite{ami1,ami3} proved that
simple descent algorithm converge to the
global estimate in these distributions. In the next section, we will show that this is
also true when using symmetry invariance constraints which are also $g$-convex.

\section{Matrix group symmetry}
In order to improve the accuracy of covariance estimators it is common to add constraints
based on prior knowledge. Of course, this priors can only be exploited if the constraints
are convex
and the associated optimization can be efficiently solved.
Recently, \cite{venkat} proposed the use of group symmetry constraints which are indeed
convex (actually linear)
and can be incorporated into a Gaussian setting. The main result in this paper is that
such sets are
also $g$-convex and can also be utilized in non-Gaussian settings.

Let $\mathcal{K}$ be a set\footnote{Here we treat the case of finite $\mathcal{K}$, but the result can be easily generalized to the infinite case.} of orthogonal matrices. Following \cite{venkat}, we formally assume that this set is actually a multiplicative group. Associated with $\mathcal{K}$, we define the fixed-point subset $\mathcal{F} \subset
\mathcal{P}(p)$ of
matrices that are invariant with respect to the conjugation by each element of
$\mathcal{K}$:
\begin{eqnarray}\label{FK}
\mathcal{F}(\mathcal{K}) = \{ \Q \in \mathcal{P}(p) | \Q = \L \Q \L^T, \forall \L \in
\mathcal{K} \}.
\end{eqnarray}

\begin{theorem} \label{main_theorem}
The set $\mathcal{F}(\mathcal{K})$ in (\ref{FK}) is $g$-convex.
\begin{proof}
First note that $\Q = \L \Q \L^T$ is equivalent to $\Q\L = \L\Q$. Now, assume $\Q_0, \Q_1 \in \mathcal{F}(\mathcal{K})$.
Let us show that the geogesic (\ref{geodesics}) lies in $\mathcal{F}(\mathcal{K})$.
Choose $\L \in \mathcal{K}$, $\L \Q_0 = \Q_0 \L, \L \Q_1 = \Q_1 \L$. Let $\M$ be a
diagonalizable matrix and $f$ a smooth function, then we can think of $f(\M)$ as of $f$ acting
on the eigenvalues of $\M$ in the orthonormal eigenbasis of $\M$. For any diagonalizable
matrix $\M$ it commutes with $\P$ iff $f(\M)$ commutes with $\P$ for any smooth function $f$,
also if two matrices $\M_1$ and $\M_2$ commute with $\P$, then their product $\M_1 \M_2$
commutes with $\P$.
This implies that $\Q_0^{-\frac{1}{2}}\Q_1\Q_0^{-\frac{1}{2}}$ commutes with $\L$,
thus $\(\Q_0^{-\frac{1}{2}}\Q_1\Q_0^{-\frac{1}{2}}\)^t$ also commutes with $\L$ and the
whole $\Q_t$ commutes with $\L$. Thus the geodesic (\ref{geodesics}) lies in
$\mathcal{F}(\mathcal{K})$ and the set
$\mathcal{F}(\mathcal{K})$ is $g$-convex.
\end{proof}
\end{theorem}

\section{Examples and Applications}

In this section we provide examples of group symmetry constraints which appear in real
world covariance estimation problems.

\subsection{Circulant}
A common class of symmetry constrained covariances is the set of positive definite
circulant matrices:
$$
\C =
 \begin{pmatrix}
  c_0 & c_1 & c_2 & \dots & c_{n-1} \\
  c_{n-1} & c_0 & c_1 & \dots & c_{n-2}\\
  \vdots & \vdots & \vdots & \ddots & \vdots \\
  c_1 & c_2 & c_3 & \dots & c_0
  \end{pmatrix}.
$$
Such matrices are typically used as an approximation to Toeplitz structured matrices which
are associated with signal processing in stationary environments  \cite{dembo},
\cite{circ}. It is easy to see that the set of circulant matrices can be expressed as
$\mathcal{F}(\mathcal{K})$ with $\mathcal{K}$ being the cyclic group of order $n$ which
acts on the rows of the
matrix by shifts. Thus, an immediate corollary of Theorem \ref{main_theorem} it that the
set of circulant matrices is $g$-convex.

%

\subsection{Persymmetric}
Another class of symmetry constrained covariances is the set of positive definite
persymmetric matrices, i.e., matrices which are symmetric in the northeast-to-southwest diagonal $\P\J_n=\J_n\P^T$, where $\J$ is the exchange $n \times n$ matrix containing ones only on the northeast-to-southwest diagonal. Since we deal with symmetric matrices the constraint becomes $\P\J_n=\J_n\P$ and the matrix form is:
$$
\P =
 \begin{pmatrix}
  p_{11} & p_{12}  & \dots & p_{1n} \\
  p_{12} & p_{22} & \dots & p_{1n-1}\\
  \vdots & \vdots & \ddots & \vdots \\
  p_{1n-1} & p_{2n-1} & \dots & p_{12}\\
  p_{1n} & p_{1n-1} & \dots & p_{11}
  \end{pmatrix}.
$$
Such matrices are commonly encountered in radar systems using a symmetrically spaced
linear array with constant pulse repetition interval \cite{pascal_per}. This structure
information could be exploited to improve detection performance \cite{dimaio},
\cite{pascal_per}. This set can be expressed as $\mathcal{F}(\mathcal{K})$ with $\mathcal{K}$ consisting of $\I_n$ and $\J_n$. Thus, an immediate corollary of Theorem \ref{main_theorem} it that the set of persymmetric matrices is also $g$-convex.

Recently, \cite{pascal_per} extended the Tyler's covariance estimator to the case of
persymmetric matrices, proposed and analyzed the asymptotic behaviour of the fixed point
estimator. Theorem \ref{main_theorem} generalizes this result to other $g$-convex
optimizations, independent of the algorithm that finds the local minimum.

%
%
%

\subsection{Proper Complex}
An important class of matrices is known as proper complex, or circularly symmetric covariance
matrices. In most radar and communication problems it is typical to work with complex
valued random variables which are invariant to rotations. A $p$-dimensional complex vector can be expressed as a $2p$-dimensional
real valued vector. Due to the symmetries, the associated
$2p\times 2p$ covariances belong to $\mathcal{F}(\mathcal{K})$ with $\mathcal{K}$ being an infinite set of rotations of the form \cite{kay}
\begin{eqnarray}
 \L_\theta=\begin{pmatrix}
  \cos\theta & \sin\theta\\
  -\sin\theta & \cos\theta  \end{pmatrix}\otimes \I_p,
\end{eqnarray}
which must hold for any $\theta$. This result already shows that the set is $g$-convex. However, in order to efficiently exploit it, we also need a finite characterization.
\begin{prop}
The set of proper complex $2p\times 2p$ covariance matrices is equivalent
to $\mathcal{F}(\mathcal{K})$ with $\mathcal{K}$ consisting of $\L_0=\I_{2p}$ and
$\L_1=\left(\begin{smallmatrix} 0&1\\ -1&0 \end{smallmatrix}\right)\otimes \I_p$.
\end{prop}
\begin{proof}
This is a particular case of the Proposition \ref{eq_quat} below.
\end{proof}

Thus, $g$-convex maximum likelihood problems with proper complex constraints can be globally and efficiently solved. As special cases this includes proper complex versions of Tyler's estimator and MGGD solutions. We note that this result is not surprising. Recently, most of these complex multivariate settings have been analyzed \cite{poor,ollila}. However, previous approaches were highly specific, and relied on defining new complex distributions. Our framework allows a unified treatment based on the real valued distributions with a single additional $g$-convex constraint.
\subsection{Proper quaternion}\label{section_0}
Another modern class of covariance matrices is known as proper quaternion \cite{quat_properness}. Quaternions are a generalization of complex numbers and is a $4$-dimensional vector space over reals, so that a length $p$ quaternion vector can be dealt with as a length $4p$ real vector. Typical applications are complex electromagnetic signals with two polarizations \cite{quaternion_polar,alba}. Similarly to the complex case, here too it is common to consider proper distributions, which are invariant to specific quaternion rotations. A $4p\times 4p$ proper quaternion covariance belongs to  $\mathcal{F}(\mathcal{K})$ with $\mathcal{K}$ being an infinite set of rotations of the form
\begin{eqnarray}\label{real_matrix}
\L_{\theta \alpha \beta \gamma}={\scriptsize{\[
\begin{array}
 {cccc} \cos({\theta}) & \alpha \sin({\theta}) & \beta \sin({\theta}) & \gamma
 \sin({\theta}) \\
 -\alpha \sin({\theta}) & \cos({\theta}) & - \gamma \sin({\theta}) & \beta \sin({\theta})
 \\
 -\beta \sin({\theta}) & \gamma \sin({\theta}) & \cos({\theta}) & -\alpha \sin({\theta})
 \\
 -\gamma \sin({\theta}) & -\beta \sin({\theta}) & \alpha \sin({\theta}) & \cos({\theta})
 \\
\end{array}
\]}}\otimes \I_p,
\end{eqnarray}
which must hold for $\theta, \alpha, \beta, \gamma$ satisfying $\alpha^2+\beta^2+\gamma^2 = 1$.
The next result characterizes this set using a finite number of constraints.
\begin{prop}\label{eq_quat}
The set of proper quaternion $4p\times 4p$ covariance matrices is equivalent
to $\mathcal{F}(\mathcal{K})$ with $\mathcal{K}$ consisting of $\L_0=\I_{4p}$, $\L_1=\R_1 \otimes \I_p$, $\L_2=\R_2 \otimes \I_p$, $\L_3=\R_3 \otimes \I_p$, $\L_4=-\L_0$, $\L_5=-\L_1$, $\L_6=-\L_2$ and $\L_7=-\L_3$, where
 \begin{equation}
\R_1=\left(\begin{smallmatrix}
0&1&0&0\\-1&0&0&0\\0&0&0&-1\\0&0&1&0\end{smallmatrix}\right),
\R_2=\left(\begin{smallmatrix}
0&0&1&0\\0&0&0&1\\-1&0&0&0\\0&-1&0&0\end{smallmatrix}\right),
\R_3=\left(\begin{smallmatrix}
0&0&0&1\\0&0&-1&0\\0&1&0&0\\-1&0&0&0\end{smallmatrix}\right).
 \end{equation}
\end{prop}
\begin{proof} The matrices $\L_i$ for $i=0,\dots,7$ are particular cases of (\ref{real_matrix}),
so the necessity is obvious. Assume now that $\Q$ is invariant under $\L_i$
conjugation, meaning that $\Q$ commutes with them: $\Q \L_i = \L_i \Q$ and we are given some matrix $\R$ of the form (\ref{real_matrix}). Take the equalities $\Q \L_i = \L_i \Q, i=0,1,2,3$,  multiply them by $\cos(\theta), \alpha \sin(\theta), \beta \sin(\theta), \gamma \sin(\theta)$
correspondingly and add them up to get: $\Q \R = \R \Q$.
\end{proof}

In other words, the set of proper quaternion covariance matrices is $g$-convex. Thus, we can easily extend the $g$-convex estimates of Tyler and MGGD to the quaternion case, and guarantee that any descent algorithm will converge to the global solution.

\section{Minimization algorithm}
In this section, we address the numerical optimization of the above minimizations. Various numerical techniques can be used to find local minimas. Since the problems are $g$-convex these local minimas will also be the global solution. The negative-log-likelihoods in (\ref{ml1})-(\ref{ml2}) have the form \cite{ami3}:
\begin{equation}
\L(\Q)=\frac{1}{n}\sum_{i=1}^n\rho(\mathbf{s}_i^T \Q^{-1} \mathbf{s}) + \log |\Q|.
\end{equation}
For simplicity, we consider the classical iterative reweighed scheme:
\begin{equation} \label{iter_scheme}
\Q_{k+1}=\frac{1}{n}\sum_{i=1}^nu(\mathbf{s}_i^T \Q^{-1} \mathbf{s}_i)\mathbf{s}_i\mathbf{s}_i^T,
\end{equation}
where $u(x)=\rho'(x)$.

Following \cite{venkat}, we note that adding the $g$-convex constraints in the form of symmetry is equivalent to replicating the sample measurements. Given $n$
$p$-dimensional measurements $\{s_i\}_{i=1}^n$ the symmetrization is equivalent to
generating synthetically $|\mathcal{K}|$ new measurements from each one, thus getting
$|\mathcal{K}|n$ samples $\{\L s_i\}_{i=1, \L\in\mathcal{K}}^n$ instead of $n$. This generalizes the iterative scheme \label{iter_scheme} as follows:
\begin{equation} \label{iter_scheme}
\Q_{k+1}=\frac{1}{|\mathcal{K}|n}\sum_{\L \in \mathcal{K}} \sum_{i=1}^nu((\L \mathbf{s}_i)^T \Q^{-1} (\L \mathbf{s}_i))(\L \mathbf{s}_i)(\L \mathbf{s}_i)^T.
\end{equation}
A simple minimization majorization argument can be used to show that this iteration leads to a descent method, see for example \cite{ami1}.

\section{Numerical Results}
For numerical simulations, we chose Tyler's scatter estimate in proper quaternion distributions. We have generated a proper real covariance
matrix $\Q_0$ and generated elliptically distributed $10$-dimensional quaternion random vectors as $s_i = \sqrt{\tau} \mathbf{v}$, where $\tau \sim \chi^2$ and $\mathbf{v}$ is zero-mean normally distributed with covariance matrix $\Q_0$. We choose $\rho(x)=p \log(x)$ to get the Tyler's covariance estimator \cite{tyler}.
\begin{center}
\includegraphics[scale=0.4]{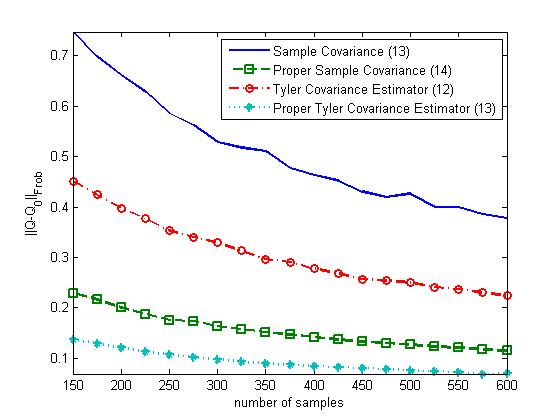}
\end{center}

We compare four different covariance estimators:
\begin{itemize}
 \item Sample Covariance
    \begin{equation}
        \Q_{SC}=\frac{1}{n}\sum_{i=1}^n \mathbf{s}_i \mathbf{s}_i^T,
    \end{equation}
 \item Proper Sample Covariance
    \begin{equation}
        \Q_{PSC} =  \frac{1}{|\mathcal{K}|n}\sum_{\L \in \mathcal{K}} \sum_{i=1}^n \L \mathbf{s}_i \mathbf{s}_i^T \L^T,
    \end{equation}
 \item Tyler Covariance Estimator Iteration
 \begin{equation}\label{tyler_iter}
      \Q_{k+1} = \frac{p}{n} \sum_{i=1}^n \frac{\mathbf{s}_i \mathbf{s}_i^T}{\mathbf{s}_i^T
      \Q_{k}^{-1} \mathbf{s}_i}.
\end{equation}
 \item Tyler Proper Covariance Estimator Iteration
    \begin{eqnarray}\label{tyler_proper_iter}
\begin{split}
& \Q_{k+1} = \frac{p}{|\mathcal{K}|n} \sum_{\L \in \mathcal{K}} \sum_{i=1}^n \frac{\L
\mathbf{s}_i \mathbf{s}_i^T \L^T}{\mathbf{s}_i^T \L^T \Q_{k}^{-1} L \mathbf{s}_i} \\
& = \frac{p}{|\mathcal{K}|n} \sum_{\L \in \mathcal{K}} \sum_{i=1}^n \frac{(\L \mathbf{s}_i) (\L \mathbf{s}_i)^T}{(\L
\mathbf{s}_i)^T \Q_{k}^{-1} (\L \mathbf{s}_i)}.
\end{split}
\end{eqnarray}
\end{itemize}

We repeat the computations for $100$ times for the four estimators with $150-600$ samples. In order to make the results consistent we divide all the matrices by their traces.

\section{Acknowledgement}
This work was partially supported by Israel Science Foundation Grant No. 786/11 and Kaete Klausner Scholarship. The authors would like to thank Alba Sloin for numerous and helpful discussions.

\end{document}

%% file: ilyas_commands.tex
\renewcommand{\(}{\left(}
\renewcommand{\)}{\right)}
\renewcommand{\[}{\left[}
\renewcommand{\]}{\right]}

\renewcommand{\L}{\mathbf{L}}

\newcommand{\R}{\mathbf{R}}

\newcommand{\0}{\mathbf{0}}

\newcommand{\I}{\mathbf{I}}
\newcommand{\J}{\mathbf{J}}
\newcommand{\C}{\mathbf{C}}
\renewcommand{\P}{\mathbf{P}}

\newcommand{\M}{\mathbf{M}}

\newcommand{\Q}{\mathbf{Q}}

\renewcommand{\log}[1]{{\rm{log}}#1}

\newtheorem{definition}{Definition}
\newtheorem{theorem}{Theorem}
\newtheorem{prop}{Proposition}